\documentclass[10pt]{article}
\usepackage{amsmath}
\usepackage{amssymb}
\usepackage{amsthm}
\usepackage{sty/preamble}
\usepackage{sty/macros}

\newtheorem{definition}{Definition}
\newtheorem{theorem}{Theorem}
\newtheorem{proposition}{Proposition}
\newtheorem{constr}{Construction}
\newtheorem{lemma}{Lemma}
\newtheorem{corollary}{Corollary}

\newcommand \com[1] {\ensuremath{\mathtt{#1}}}

\begin{document}


\title{Estimating the hardness of SAT encodings for Logical Equivalence Checking of Boolean circuits}


\author{Alexander Semenov* \and Konstantin Chukharev \and  Egor Tarasov \and Daniil Chivilikhin \and Viktor Kondratiev
\\
ITMO University, St. Petersburg, Russia\\
\href{mailto: alex.a.semenov@itmo.ru}{\texttt{ alex.a.semenov@itmo.ru}}}

\maketitle              

\begin{abstract}
In this paper we investigate how to estimate the hardness of Boolean satisfiability (SAT) encodings for the Logical Equivalence Checking problem (LEC). Meaningful estimates of hardness are important in cases when a conventional SAT solver cannot solve a SAT instance in a reasonable time.
We show that the hardness of SAT encodings for LEC instances can be estimated \textit{w.r.t.} some SAT partitioning.
We also demonstrate the dependence of the accuracy of the resulting estimates on the probabilistic characteristics of a specially defined random variable associated with the considered partitioning.
The paper proposes several methods for constructing partitionings, which, when used in practice, allow one to estimate the hardness of SAT encodings for LEC with good accuracy.
In the experimental part we propose a class of scalable LEC tests that give extremely complex instances with a relatively small input size $n$ of the considered circuits. For example, for $n = 40$, none of the state-of-the-art SAT solvers can cope with the considered tests in a reasonable time.
However, these tests can be solved in parallel using the proposed partitioning methods.
\end{abstract}

\section{Introduction}
Boolean circuits are widely used in theoretical computer science~\cite{arora2009,goldreich2008} as well as in numerous industrial applications.
It would take too much space to list all the key references regarding the various practical applications of Boolean circuits.
We only note that each hardware implementation of an arbitrary discrete function (\ie function $f \colon \{0, 1\}^n \to \{0, 1\}^m$) can be viewed as some Boolean circuit, entailing the development of such a colossal industry as Electronic Design Automation~(EDA).

One of the main problems related to Boolean circuits is the logical equivalence checking problem (LEC)~\cite{kuehlmann1997,molitor2004}.
This problem is posed as follows: there are two circuits $S_f, S_h$ specifying some functions $f,h \colon \{0, 1\}^n \to \{0, 1\}^m$.
The question is: \enquote{Is it true that $f$ and $h$ are equal, \ie point-wise equality $f \cong h$ holds?}.
At the initial stage of development of formal verification methods, Binary Decision Diagrams (BDD)~\cite{bryant1986} were used to solve LEC.
Works~\cite{biere1999,biere1999a} argued in favor of solving LEC via applying complete SAT solvers based on the CDCL algorithm~\cite{marques-silva2009};
currently, LEC is mainly solved with such algorithms: a good example is the ABC~\cite{BraytMisch10} framework.

SAT solvers work with Boolean formulas in Conjunctive Normal Form (CNF).
There is an algorithm linear in the size of circuits $S_f,S_h$ that reduces LEC for these circuits to SAT for a CNF formula using Tseytin transformations~\cite{tseytin1970}.

Unfortunately, SAT for a CNF formula which encodes LEC for $S_f$ and $S_h$ can be difficult for state-of-the-art SAT solvers.
If we use a sequential solver, in many cases we cannot even say how much time can be required for solving the corresponding SAT instance.
Prediction of runtime for modern SAT solvers is very difficult in the general case due to their heavy-tailed behavior~\cite{GomesSabh09}.

The main goal of this paper is to show that the hardness of a SAT instance which 
encodes some LEC problem can be estimated by decomposing this instance into
a family of simpler SAT instances.
In this context we introduce the notion of hardness of formula w.r.t.
some SAT partitioning. 
We show that this hardness measure can be expressed via an expected
value of a special random variable which is associated with a
considered SAT partitioning.
To estimate this measure we use the Monte Carlo method. 
The main issue of this approach is that the corresponding
Monte Carlo estimation can be not accurate enough. 
We study the problem how to construct a partitioning of a CNF formula 
encoding some LEC problem, which gives a hardness estimation of this
formula with high accuracy.
We propose two partitioning construction methods which 
rely on the structure of considered circuits
and justify the good properties of proposed construction methods
in application to extremely hard LEC instances.
In particular, using a computing cluster we solved 
the LEC instance which turned to be too hard for sequential SAT solvers
which won the SAT Competitions of the last years.

\section{Preliminaries}
In this section, we introduce the necessary formal concepts and notation.

\subsection{Satisfiability and Boolean circuits}

We start from basic concepts related to SAT, the Boolean Satisfiability problem~\cite{HandBook09}.
In the context of SAT one usually works with a Boolean formula in CNF.

Let $C$ be an arbitrary CNF formula and $X$ be the set of Boolean variables occurring in~$C$.
An \emph{assignment} of variables from $X$ is a mapping $\alpha \colon X \rightarrow \{0,1\}$.
The set of all different assignments of variables from~$X$ is denoted as $\{0,1\}^{|X|}$ and called \emph{Boolean hypercube} of dimension $n$, $n = |X|$.

In the context of SAT, for an arbitrary CNF formula~$C$ it is required to answer the following question: is it true that $C$ is satisfiable?
That is, is there as assignment of variables from $X$ for which $C$ is evaluated to \emph{true}?
In this formulation, SAT is NP-complete, and it is NP-hard when one has to detect the satisfiability of $C$ and, in the case of a positive answer, to find some satisfying assignment.
Despite the theoretical hardness of SAT, the last 20 years demonstrate impressive progress in the development of SAT solving algorithms with a wide spectrum of practical applications in symbolic verification, computational combinatorics, bioinformatics, cryptanalysis, etc.
One of the most striking examples is hardware verification and, in particular, Logical Equivalence Checking~(LEC).
As it was said above, in LEC one has to answer the following question: is it true that two Boolean circuits are equivalent?

As in the majority of related articles, we regard a \emph{Boolean circuit} as some directed acyclic graph.
Consequently, we use the following standard graph theory definitions.
A (directed) \emph{graph} $G = (V,E)$ consists of a set of vertices~$V$ and a set of (directed) edges~$E \subseteq V^2$.
An~\emph{edge} is a pair of connected vertices.
An~\emph{arc} is a directed edge, \ie an ordered pair of vertices.
For each arc $(u,v) \in E$, vertex~$u$ is called a \emph{parent} of~$v$, and $v$~is called a \emph{child} of~$u$.
The set of all parents of a vertex~$v$ is denoted by~$P_v$.
The \emph{indegree} of a vertex~$v$ is the number of parents of~$v$, and the \emph{outdegree} is the number of children.
A~vertex is called an \emph{input} if it has no parents, and an \emph{output} if it has no children.
The sets of inputs and outputs are denoted as $V^{in} \subset V$ and $V^{out} \subset V$ respectively.
A~\emph{path} is a sequence of arcs.
A~vertex~$u$ is called a \emph{predecessor} of~$v$ if there is a path from $u$ to~$v$.
A~predecessor~$u$ which is also an input ($u \in \Vin$) is called an \emph{ancestor} of~$v$.
The set of all ancestors of a vertex~$v$ is denoted by~$A_v$ (\emph{ancestor set}).

A Boolean circuit with $n$ inputs and $m$ outputs can be viewed as a natural way of specifying some discrete function $f \colon \{0, 1\}^n \to \{0, 1\}^m$.
Implying this, we will denote an arbitrary Boolean circuit defining a discrete function $f$ as $S_f = (V,E)$.

Let $S_f$ be an arbitrary Boolean circuit.
Any vertex $v \in V \setminus V^{in}$ is called a \emph{gate}.
Each gate is associated with some \emph{logical connective} from a predefined set called a \emph{basis} (for example it can be $\{\land, \lnot\}$, $\{\land, \lor, \lnot\}$, $\{\land, \xor, 1\}$, etc.).
An example of a graphical representation of a Boolean circuit with $|V^{in}| = 3$ inputs and $|V \setminus V^{in}| = 8$ gates is shown in~\cref{fig:boolean-circuit-example}.

\begin{figure}[!htb]
    \centering
    \begin{adjustbox}{max width=\linewidth}
    
\begin{tikzpicture}[
    auto,
    on grid,
    >={Stealth[]},
    gate/.style={
        draw,
        circle,
        minimum size=6mm,
    },
]
    \newcounter{mynode}
    \newcommand{\nextnode}{%
        \stepcounter{mynode}%
        \arabic{mynode}%
    }
    
    
    \node[gate, label=center:$i_1$] (i1) {};
    \node[gate, label=center:$i_2$] (i2) [right=of i1] {};
    \node[gate, label=center:$i_3$] (i3) [right=of i2] {};
    \node[gate, label=center:$\neg$] (g1-1) [below=of i1] {};
    \node[gate, label=center:$\land$] (g1-2) [below=of i2] {};
    \node[gate, label=center:$\xor$] (g1-3) [below=of i3] {};
    \node[gate, label=center:$\lor$] (g2-2) [below=of g1-2] {};
    \node[gate, label=center:$\land$] (g2-3) [below=of g1-3] {};
    \node[gate, label=center:$\land$] (g3-1) [below=of $(g1-1 |- g2-2)!0.5!(g2-2)$,anchor=center] {};
    \node[gate, label=center:$\land$] (g3-2) [below=of $(g2-2)!0.5!(g2-3)$,anchor=center] {};
    \node[gate, label=center:$\xor$] (o1) [below=of $(g3-1)!0.5!(g3-2)$,anchor=center] {};
    
    \draw[<-]
        (o1) edge (g3-1) edge (g3-2)
        (g3-1) edge (g1-1) edge (g2-2)
        (g3-2) edge (g2-2) edge (g2-3)
        (g2-2) edge (g1-2) edge (g1-3)
        (g2-3) edge (g1-2) edge (g1-3)
        (g1-1) edge (i1)
        (g1-2) edge (i1) edge (i2)
        (g1-3) edge (i2) edge (i3)
    ;

\end{tikzpicture}%

    \end{adjustbox}%
    \caption{An example Boolean circuit with three inputs ($i_1$,~$i_2$,~$i_3$) and eight gates}
    \label{fig:boolean-circuit-example}
\end{figure}
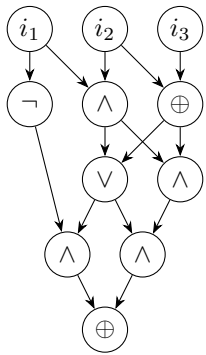

The set of vertices $V$ of a circuit $S_f$ can be naturally partitioned into subsets called \enquote{layers}, which are defined inductively as follows.

\begin{definition}[Circuit layers]
Let $V_0 = V^{in}$ denote the zeroth circuit layer.
The $k$-th ($k \geq 1$) circuit layer $V_k$ is defined inductively as the set of all vertices~$v$ satisfying the following two properties:
\begin{enumerate}[left=\parindent, topsep=4pt, itemsep=4pt]
    \item $v \notin \bigunionnolim{j=0}^{k-1} V_j$;
    \item $P_v \subseteq \bigunionnolim{j=0}^{k-1} V_j$.
\end{enumerate}
\end{definition}

\begin{definition}[Associated functions]
With each gate $v \in V \setminus V^{in}$ let us associate a predefined Boolean function $g_v \colon \{0,\! 1\}^{|P_v|} \rightarrow\nobreak \{0,\! 1\}$.
The value of $g_v$ is uniquely determined by the values of the functions $g_w$ ($w \in P_v$) with respect to the semantics of the logical connective which is associated with gate $v$.
\end{definition}

Let us fix some order on set $V^{in}$ and the same order will apply to the bits of an arbitrary word from $\{0,1\}^n$.
Thus, each bit of an arbitrary word $\alpha \in \{0,1\}^n$ is uniquely connected with some vertex from $V^{in}$.
Let us say that $\alpha$ is an \emph{input word} of $S_f$.

\begin{definition}[Circuit interpretation]
Let $\alpha \in \{0, 1\}^n$ be an arbitrary input word of the circuit~$S_f$.
Begin traversing the circuit starting from the first layer~$V_1$.
For any $v \in V_1$ we suppose that the value of an arbitrary $g_w$ ($w \in P_v$, $P_v \subseteq V^{in}$) is equal to the corresponding bit of $\alpha$ associated with $w$.
For an arbitrary gate $v \in V_j$, $j > 1$, let us calculate the value of $g_v$ on $\alpha$ using known values of $g_w$ on this input word for all $w \in P_v$.
We will also say that this value of $g_v$ is \emph{induced} by~$\alpha$.
Continue the evaluation until the values of functions $g_v$ are calculated for all gates of circuit~$S_f$.
Let us call the described process the \emph{interpretation} of the circuit $S_f$ on input word~$\alpha$.
\end{definition}

Let $S_f$ be a Boolean circuit with $n$ inputs and $m$ outputs.
Note that the interpretation of $S_f$ specifies a total function $f \colon \{0,1\}^n \rightarrow \{0,1\}^m$.
The value of this function on an arbitrary word $\alpha \in \{0,1\}^n$ is a Boolean vector $\gamma = (\gamma_1, \dotsc, \gamma_m)$, where $\gamma_k$, $k \in \{1, \dotsc, m\}$, are the values of functions $g_v$ induced by $\alpha$ for all $v \in V^{out}$.

\begin{definition}[Associated variables]
Let us associate with each vertex of circuit $S_f$ a particular Boolean variable and denote the set of all such variables as $X$. 
Let $X^{in}$ be the set of Boolean variables \emph{associated} with the inputs of $S_f$; we will
refer to these variables as to \emph{input variables}. 
The variables assigned to gates will be called \emph{auxiliary variables}.
For an arbitrary $\widetilde{V} \subseteq V$, let $\var(\widetilde{V})$ denote the set of Boolean variables assigned to nodes from~$\widetilde{V}$.
To simplify the notation, we write $\var(v) = x$ for a singleton vertex instead of $\var(\Set{v}) = \Set{x}$.
\end{definition}

Let $v$ be an arbitrary gate in~$S_f$, and let $U_v = \var(P_v)$, $u = \var(v)$.
Let $g_v$ be a Boolean function corresponding to the gate~$v$, and let $F(g_v)$ be an arbitrary Boolean formula over~$U_v$ (for example, a canonical CNF), which defines~$g_v$.
For a gate~$v$, we denote by~$C_v$ the CNF representation of  formula $F(g_v) \equiv u$.

\begin{definition}[Template CNF formula]
Let $S_f$ be some Boolean circuit which specifies the function $f \colon \{0,1\}^n \rightarrow \{0,1\}^m$.
We will refer to the CNF formula $C_f = \biglandnolim{v \in V \setminus V^{in}} C_v$ as to the \emph{template CNF formula} for function $f$ \cite{semenov2020}.
\end{definition}

\noindent Note that $C_f$ is in fact the CNF formula which can be obtained from~$S_f$ by applying Tseytin transformations~\cite{tseytin1970}.

Recall the following notation:
\(
    x^\sigma = \bigl\{\begin{smallmatrix*}[l]
        \hfill      x, &\textnormal{if } \sigma = 1 \\
        \hfill \neg x, &\textnormal{if } \sigma = 0
    \end{smallmatrix*}
\).
Let $\Phi$ be an arbitrary Boolean formula over the variables~$X$.
Denote by~$\Phi|_{x = \sigma}$ the formula obtained by substituting $x$ with~$\sigma$ in~$\Phi$~\cite{chang1973}.
It is clear that the formulas $x^\sigma \land \Phi$ and $\Phi|_{x = \sigma}$ are equisatisfiable.
Thus, when working with the formula $x^\sigma \land \Phi$, we can regard the unit clause~$x^\sigma$ as the value~$\sigma$ of the variable~$x$ in the formula~$\Phi$.

\begin{definition}[Cone]
Let $v$ be an arbitrary gate in~$S_f$, and $Q_v$ be the set of all predecessors of~$v$.
The set $Q_v \union \Set{v}$ is called the \emph{cone} of~$v$.
\end{definition}

The following fact has been repeatedly established in the literature, see \textit{e.g.}~\cite{bessiere2009,drechsler2009}.
It uses a simple Boolean constraint propagation mechanism known as the Unit Propagation rule (UP)~\cite{dowling1984,marques-silva2009}.

\begin{lemma}\label{lem:first}
Let~$C_f$ be the template CNF formula for a circuit~$S_f$.
Let~$v$ be an arbitrary gate of $S_f$, the set $Q_v$ be the cone of~$v$, $A_v$, $A_v \subseteq Q_v$ be the ancestor set of~$v$, and $X_v = \Set{x_{v,1},\dotsc,x_{v,r}} = \var(A_v)$ be the set of variables associated with $A_v$ ($X_v \subseteq X^{in}$).
Then, for each $\Tuple{\alpha_1,\dotsc,\alpha_r} \in \{0, 1\}^{|X_v|}$, application of the UP rule to the CNF formula $x_{v,1}^{\alpha_1} \land\ldots\land x_{v,r}^{\alpha_r} \land C_f$ derives (in the form of unit clauses) the values of all variables from $var(Q_v) \setminus X_v$.
Moreover, for the variable $u = var(v)$, the derived value is equal to the value of function $g_v$ induced by any input word $\alpha \in \{0, 1\}^n$ of $S_f$ which contains
(\textit{w.r.t.} corresponding variables) the sub-vector $\Tuple{\alpha_1,\dotsc,\alpha_r}$.
Note that the resulting set of unit clauses does not contain conflicting literals.
\end{lemma}

\begin{proof}
The proof of this lemma uses the traversal of~$S_f$ by layers and the properties of Tseytin transformations.
\end{proof}

\begin{corollary}[of~\cref{lem:first}]
\label{cor:first}
Application of UP to the CNF formula
$x_1^{\alpha_1} \land\ldots\land x_n^{\alpha_n} \land C_f$ for any $\Tuple{\alpha_1,\dotsc,\alpha_n} \in \{0, 1\}^{|X^{in}|}$ derives (in the form of unit clauses) the values of all variables associated with gates from~$V \setminus V^{in}$, including the variables from $var(V^{out}) = \{y_1, \dotsc, y_m\}$:
$y_1=\gamma_1, \dotsc, y_m=\gamma_m$,
$f(\alpha) = \gamma$,
$\alpha = \Tuple{\alpha_1, \dotsc, \alpha_n}$,
$\gamma = \Tuple{\gamma_1, \dotsc, \gamma_m}$.
\end{corollary}

Note that from \cref{lem:first} and \cref{cor:first} it follows that the process of interpretation of circuit $S_f$ on an arbitrary input word $(\alpha_1, \dotsc, \alpha_n)$ is modelled by consecutive application of the UP rule to the CNF formula $x_1^{\alpha_1} \land \ldots \land x_n^{\alpha_n} \land C_f$.

\subsection{SAT partitioning}

As mentioned above, SAT is NP-hard, so some instances of SAT can be very difficult for conventional solvers.
There are several approaches to parallelizing SAT solving~\cite{Hyvarinen2011}, the main ones being the portfolio approach (\textit{e.g.}~\cite{hordesat}) and the partitioning approach~(\textit{e.g.}~Cube and Conquer~\cite{CC2012}).
In this paper, we follow the partitioning approach.

Let us consider an arbitrary CNF formula~$C$ over the set of Boolean variables~$X$ and set $\Pi = \{G_1, \dotsc, G_s\}$, where $G_i$ ($i \in\nobreak \{1, \dotsc, s\}$) are some Boolean formulas.
Let~us say that the set $\Pi$ yields a SAT partitioning of~$C$ if the following conditions hold:
\begin{itemize}
    \item formulas $C$ and $C \land (G_1 \vee \dots \vee G_s)$ are equisatisfiable;
    \item for each $i, j \in \{1, \dotsc, s\}$, $i \neq j$, formula $C \land G_i \land G_j$ is unsatisfiable.
\end{itemize}

For some set of variables $B \subseteq X$, $B = \{x_{k_1}, \ldots, x_{k_r}\}$, each formula $x_{k_1}^{\alpha_1}, \ldots x_{k_r}^{\alpha_r}$ for an arbitrary $(\alpha_1, \ldots, \alpha_r) \in \{0, 1\}^r$ is called a \emph{cube} (over~$X$).
For an arbitrary CNF formula $C$ over the set of variables~$X$, a simple
example of a partitioning is generated by set $\Pi = \{G_1,\ldots, G_{2^r}\}$, which consists of all possible cubes over an arbitrary set $B$, $|B| = r$.

In the following, we will use the term \emph{partitioning} for both the set~$\Pi$ and for the set of CNF formulas generated by~$\Pi$.

\subsection{Background from probability theory}

Below we will use some probabilistic reasoning to estimate the hardness of SAT encodings for LEC
instances.
Let us recall some relevant basic facts from probability theory.

Let $\xi$ be some random variable with finite spectrum (i.e. set of its values) $\{\xi_1, \ldots, \xi_M\}$ and probability distribution $P_\xi = \{p_1, \ldots, p_M\}$.
In the following, use assume that $0 < \xi_i < \infty$ for every $i \in \{1, \ldots, M\}$.
Then, the expected value (expectation) of $\xi$ is defined as $E[\xi] = \bigsumnolim{i = 1}^{M} \xi_i p_i$.
In~many practical applications, knowledge of $E[\xi]$ turns out to be very important.
However, it is often impossible to accurately calculate the exact value of $E[\xi]$ in a reasonable amount of time.
In such cases, one can instead estimate $E[\xi]$ with some predetermined accuracy $\varepsilon$.
The corresponding algorithms use random sampling and traditionally refer to the Monte Carlo method~\cite{MetropUlam49}.

More precisely, let $\xi_1, \ldots, \xi_N$ be independent observations of the random variable $\xi$.
Then, Chebyshev’s inequality~\cite{Feller71} implies:
\begin{equation}
    \label{eq:cheb}
    \Pr \left\{
        (1 \!-\! \varepsilon) \! \,\! E[\xi]
        \leq
        \frac{1}{N} \!\!\bigsum{j = 1}^N \! \xi_j
        \leq
        (1 \!+\! \varepsilon) \! \,\! E[\xi]
    \right\} \!\geq\!
    1 - \frac{Var(\xi)}{\varepsilon^2 \! \,\! N \! \,\! E^2[\xi]},
\end{equation}
where $Var(\xi)$ denotes the variance of the random variable~$\xi$.
It follows from~\eqref{eq:cheb} that for finite $E[\xi]$ and~$Var(\xi)$, the expectation $E[\xi]$ can be approximated (in the sense of~\eqref{eq:cheb}) by the value $\hat{\xi}=\frac{1}{N} \bigsumnolim{j = 1}^N \xi_j$ with any tolerance~$\varepsilon$ given in advance by increasing the number of observations~$N$.

The following inequality is less accurate than~\eqref{eq:cheb} in the general case, but it is often more convenient to use for $E[\xi]$ estimation:
\begin{equation}
    \label{eq:cheb-2}
    \Pr \left\{
        \left| E[\xi] - \frac{1}{N} \bigsum{j = 1}^N \xi_j \right| \leq \varepsilon \right
    \} \geq
    1 - \frac{Var(\xi)}{\varepsilon^2  \, N}.
\end{equation}
If~$\xi$ is a Bernoulli variable, \ie takes values in~$\{0,1\}$, then instead of~\eqref{eq:cheb-2} one can use the following variant of Chernoff bound (see \textit{e.g.}~\cite{MotwRagh95}):
\begin{equation}
    \Pr \left\{
        \left| E[\xi] - \frac{1}{N} \bigsum{j = 1}^N \xi_j \right| \leq \varepsilon
    \right\} \geq
    1 - 2  \, e^{ - \frac{\varepsilon^2  \, N}{4}}.
\label{eq:cheb-3}
\end{equation}
The proof of~\eqref{eq:cheb-3} can be found in~\cite{KarpLuby89}.

\section{Estimating the hardness of SAT encodings of LEC instances using SAT partitioning}

Let us return to LEC.
Consider two Boolean circuits $S_f$, $S_h$ defining functions $f, h \colon \{0,1\}^n \rightarrow \{0,1\}^m$.
Let us construct a circuit which will be denoted by $S_{f \glue h}$.
This circuit is obtained from $S_f$ and $S_h$ by \enquote{gluing} together the input vertices (see~\cref{fig:glued}).
Thus, this circuit has the same $V^{in}$ as $S_f$ and $S_h$, and defines the function $f \glue h: \{0,1\}^n \rightarrow \{0,1\}^{2m}$ .

\begin{figure}
    \centering
    \begin{tikzpicture}[
    auto,
    on grid,
    >={Stealth[]},
    dot/.style={
        draw,
        fill,
        circle,
        minimum size=4pt,
        inner sep=0pt,
    },
    smalldot/.style={
        draw,
        fill,
        circle,
        minimum size=2pt,
        inner sep=0pt,
    },
]
    \def\HPosF{2}
    \def\VPosF{1.6}
    \def\VGapIO{0.3}

    \node[dot]      (X1) at (  -1, 0) {};
    \node[smalldot] (X2) at (-0.5, 0) {};
    \node[smalldot] (X3) at (   0, 0) {};
    \node[smalldot] (X4) at ( 0.5, 0) {};
    \node[dot]      (X5) at (   1, 0) {};
    
    \draw [decorate, decoration={brace,raise=6pt,amplitude=6pt}] (-1,0) -- (1,0) node[pos=0.5,above=10pt] {$n$ inputs};
    
    \node[draw] (F1) at (-\HPosF,-\VPosF) {$S_f : \{0,1\}^n \to \{0,1\}^m$};
    \node[draw] (F2) at ( \HPosF,-\VPosF) {$S_h : \{0,1\}^n \to \{0,1\}^m$};
    
    \foreach \k in {1,2} {
        \coordinate (F\k i1) at ([xshift=.5\pgflinewidth]F\k.north west);
        \coordinate (F\k i5) at ([xshift=-.5\pgflinewidth]F\k.north east);
        \coordinate (F\k i3) at (F\k.north);
        \coordinate (F\k i2) at ($(F\k i1)!0.5!(F\k i3)$);
        \coordinate (F\k i4) at ($(F\k i3)!0.5!(F\k i5)$);
        
        \coordinate (F\k o1) at ([xshift=.5\pgflinewidth]F\k.south west);
        \coordinate (F\k o5) at ([xshift=-.5\pgflinewidth]F\k.south east);
        \coordinate (F\k o3) at (F\k.south);
        \coordinate (F\k o2) at ($(F\k o1)!0.5!(F\k o3)$);
        \coordinate (F\k o4) at ($(F\k o3)!0.5!(F\k o5)$);
        
        \foreach \i in {1,...,5} {
            \node[dot,fill=white] [above=\VGapIO of F\k i\i] (F\k I\i) {};
            \node[dot,fill=white] [below=\VGapIO of F\k o\i] (F\k O\i) {};
        }
    }
    
    \foreach \k in {1,2} {
        \foreach \i in {1,...,5} {
            \foreach \t/\T in {i/I,o/O} {
                \draw (F\k\t\i) -- (F\k\T\i);
            }
        }
    }
    
    \draw (X1) -- (F1I1);
    \draw (X1) -- (F2I1);
    \draw (X5) -- (F1I5);
    \draw (X5) -- (F2I5);
    \foreach \k in {1,2} {
        \foreach \i in {2,...,4} {
            \draw[dashed,lightgray] (X\i) -- (F\k I\i);
        }
    }
    
    \draw [decorate, decoration={brace,mirror,raise=6pt,amplitude=6pt}] (F1O1) -- (F2O5) node[pos=0.5,below=10pt] {$2m$ outputs};

\end{tikzpicture}%
    \caption{Circuit $S_{f \glue h}$ constructed by using the same set of inputs for two circuits $S_f$ and $S_h$}
    \label{fig:glued}
\end{figure}
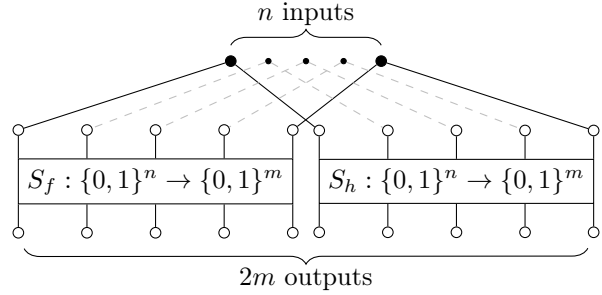

Denote $V_f^{out}$ and $V_h^{out}$ the output sets of circuits $S_f, S_h$, and denote $Y_f = \Set{y_1^f, \dotsc, y_m^f}$ and $Y_h = \Set{y_1^h, \dotsc, y_m^h}$ the sets of variables assigned to vertices from $V_f^{out}$ and $V_h^{out}$ and ordered according to the semantics of the circuits.
Consider the formula:
\begin{equation}
\label{eq:miter-func}
\left(y_1^f \xor y_1^h\right) \lor \cdots \lor \left(y_m^f \xor y_m^h\right).
\end{equation}
Formula~\eqref{eq:miter-func} defines a Boolean function $\mu \colon \{0, 1\}^{2m} \to \{0, 1\}$ called a \emph{miter}.
Let us apply Tseytin transfomation (in a standard manner) to formula \eqref{eq:miter-func} and denote the resulting CNF formula as $C(\mu)$.
It follows directly from~\cref{lem:first} that circuits $S_f$ and $S_h$ are equivalent if and only if the following CNF formula is unsatisfiable:
\begin{equation}
\label{eq:miter-cnf}
    C_{f \glue h} \land C(\mu),
\end{equation}
where $C_{f \glue h}$ is the template CNF formula for function~$f \glue h$.
Consider below another corollary of~\cref{lem:first}.

\begin{corollary}
For two arbitrary functions $f, h \colon \{0,1\}^n \rightarrow \{0,1\}^m$ specified by circuits $S_f, S_h$, the number of assignments satisfying template CNF formulas $C_f$, $C_h$, and $C_{f \glue h}$, is~$2^n$.
\label{cor:second}
\end{corollary}

A rather interesting observation is the following.
Modern CDCL-based SAT solvers, receiving a CNF formula of the form $C_f$ or $C_{f \glue h}$ as input, very quickly (usually within a fraction of a second) generate some satisfying assignment corresponding to some input/output pair.
At the same time, CNF formulas~\eqref{eq:miter-cnf} can be very hard.
It is worth to note that formulas $C_{f \glue h}$ and~\eqref{eq:miter-cnf} differ from each other only by clauses corresponding to the miter (their fraction in~\eqref{eq:miter-cnf} usually is extremely small).

Everywhere below, we assume that $\Oracle$ is an arbitrary complete SAT solver.
If formula~\eqref{eq:miter-cnf} is difficult for $\Oracle$, then often we cannot even say how much time $\Oracle$ will take to solve this SAT instance.
The difficulty of this kind of assessment is due to an effect that is known as the heavy-tailed behavior of CDCL-based SAT solvers~\cite{GomesSabh09}.
However, in some cases, we can estimate the overall hardness of a SAT instance quite efficiently and accurately by estimating the hardness of its SAT partitioning. 
Let us note that the following definition is inspired by the decomposition hardness notion \cite{CP2021}.

\begin{definition}[Hardness w.r.t. SAT partitioning]
Let $C$ be an arbitrary CNF formula, $\Pi$ be some partitioning of $C$, and $\Oracle$ be some complete SAT solver.
The total runtime of solver $\Oracle$ on instances $G \land C$ for all formulas $G \in \Pi$ is called the \emph{hardness} of $C$ w.r.t solver $\Oracle$ and partitioning $\Pi$, and is denoted as $T_\Oracle(C, \Pi)$.
\end{definition}

Below we show that $T_\Oracle(C, \Pi)$ can be estimated using a probabilistic Monte Carlo algorithm.
Let us describe the general scheme for constructing such estimates.

Let $C$ is an arbitrary SAT instance and $\Pi = \{G_1, \ldots, G_s\}$ be some partitioning of $C$.
If $s$ is large, then it is possible to estimate the time needed to solve~$C$ \textit{w.r.t.}~$\Pi$ through an estimate of the solution time of~$N$ SAT instances $G_k \land C$, $k \in \{1, \ldots, N\}$ chosen from~$\Pi$ according to some distribution.
As a rule, at the initial stage we fix a uniform distribution on~$\Pi$.
Let us introduce a random variable~$\xi_\Pi$ whose values are equal to the running time of the SAT solver~$\Oracle$ on formulas $G_j \land C$, $j \in \{1, \dotsc, s\}$.
Let~$Spec(\xi_\Pi) = \{\xi_1, \dotsc, \xi_Q\}$ be the spectrum of $\xi_\Pi$, and each value $\xi_r$, $r \in \{1, \dotsc, Q\}$, is assigned a probability $p_r = \frac{\#\xi_r}{s}$,
where $\#\xi_r$ denotes the number of such $G_j$, $j \in \{1, \dotsc, s\}$, that the running time of $\Oracle$ on the formula $G_j \land C$ is~$\xi_r$.
Thus,~$\xi_\Pi$ has the distribution law $P(\xi_\Pi) = \{p_1, \ldots, p_Q\}$.
Recall again that $\Oracle$ is complete SAT solver, so $\xi_\Pi$ has finite spectrum, expected value, and variance.
The following fact is true.

\begin{theorem}
The hardness of SAT instance $C$ w.r.t. solver $\Oracle$ and SAT partitioning $\Pi$ is $T_\Oracle(C, \Pi) = s  \, E[\xi_\Pi]$.
\end{theorem}
\begin{proof}
\[
    T_\Oracle(C, \Pi)
    = \bigsum{r=1}^Q \xi_r  \, \#\xi_r
    = s  \, \bigsum{r=1}^Q \xi_r  \, \frac{\#\xi_r}{s} = s  \, E[\xi_\Pi].
\]
\end{proof}

The running time of~$\Oracle$ can be measured in any convenient units, for example, in seconds, the number of times the Unit Propagation rule is applied, or the number of conflicts generated by~$\Oracle$.

To estimate $E[\xi_\Pi]$, one can use the Monte Carlo method and specifically the formula~\eqref{eq:cheb}.
Despite the formal possibility of achieving any estimation accuracy by increasing the number of observations~$N$ of the value~$\xi_\Pi$, in many practical cases, the obtained estimates may be inaccurate due to high variance~$Var(\xi_\Pi)$, which, in turn, is a consequence of the effect of heavy-tailed behavior of CDCL SAT solvers.
Thus, arises the problem of constructing such partitionings for which $Var(\xi_\Pi)$ would not exceed some reasonable limit:
for example, the standard deviation $\sigma = +\sqrt{Var(\xi_\Pi)}$ should not exceed $E[\xi_\Pi]$.
Below we describe two general SAT partitioning constructions for which $\sigma$ has relatively small values on the LEC instances discussed below.
The ideas underlying such constructions are based on the properties of CNF formulas $C_f$, $C_h$, $C_{f \glue h}$, and $C_{f \glue h} \land C(\mu)$.

Consider LEC for circuits $S_f, S_h$ ($f, h \colon \{0, 1\}^n \rightarrow \{0, 1\}^m$).
Let us once again focus on the fact that the CNF formula $C_{f \glue h}$ has $2^n$ satisfying assignments,
while the CNF formula $C_{f \glue h} \land C(\mu)$ has none if $S_f \cong S_h$.
Let $\Pi = \{G_1, \ldots, G_s\}$ be an arbitrary SAT partitioning of $C_{f \glue h}$.
Denote by $\#(G \land C)$ the number of satisfying assignments of the formula $G \land C$ for an arbitrary $G \in \Pi$.
It is easy to deduce the following fact from the general properties of the SAT partitioning and~\cref{lem:first}.

\begin{proposition}
\label{cor:prop2}
Let $\Pi$ be an arbitrary SAT partitioning of $C_{f \glue h}$. Then, the following equation holds:
\[
    \bigsum{G \in \Pi} \#(G \land C_{f \glue h}) = 2^n.
\]
\end{proposition}

Thus, an arbitrary formula $G \land C_{f \glue h}$ has $\#(G \land C_{f \glue h})$ satisfying assignments (and this number can be significantly larger than zero),
but at the same time the formula $G \land C_{f \glue h} \land C(\mu)$ is unsatisfiable if $S_f$ and $S_h$ are equivalent.
Allowing a somewhat loose interpretation, we can say that by proving the unsatisfiability of $G \land C_{f \glue h} \land C(\mu)$,
we \emph{block} $\#(G \land C_{f \glue h})$ satisfying assignments of formula $G \land C_{f \glue h}$.
In total, when solving all problems in the SAT partitioning, we need to block all $2^n$ satisfying assignments of $G \land C_{f \glue h}$
(the sets of assignments of different formulas $G \land C_{f \glue h}$ are disjoint).

Taking into account all said above, there arises an attractive idea to link the hardness of formulas  $G \land C_{f \glue h} \land C(\mu)$
with the number of satisfying assignments of corresponding formulas $G \land C_{f \glue h}$.
Looking ahead, let us note that our computational experiments demonstrate this exact connection: the more satisfying assignments the formula $G \land C_{f \glue h}$ has,
the harder the formula $G \land C_{f \glue h} \land C(\mu)$ is.
Thus, if we want all problems $G \land C_{f \glue h} \land C(\mu)$ in the SAT partitioning $\Pi$ to have approximately equal hardness
(which would correspond to a relatively low variance $Var(\xi_\Pi)$), we must ensure that all formulas  $G \land C_{f \glue h}$
have an approximately equal number of satisfying assignments.
In the following, we describe two types of SAT partitionings that satisfy these requirements.

\begin{constr}
Denote as $\widetilde{C}$ the CNF formula $C_{f \glue h}$ or CNF formula $C_{f \glue h} \land C(\mu)$.
Consider the set of variables $X^{in} = \{x_1, \ldots, x_n\}$ assigned to the inputs of the circuit $S_{f \glue h}$.
Let us divide (generally speaking, in an arbitrary way) the set $X^{in}$ into disjoint subsets of variables with $k$, $k \geq 1$ variables in each group.
For simplicity, we will assume that $n$ is divisible by~$k$.
We have sets $X^1, \dotsc, X^{n/k}$.
With each set $X^j$, $j \in \{1, \dotsc, n/k\}$ we associate an arbitrary non-constant Boolean function $\lambda_1^j$, $\lambda_1^j \colon \{0, 1\}^{|X^j|} \rightarrow \{0, 1\}$,
and the function ${\lambda_2^j = \lnot \lambda_1^j}$.
Let~$\phi_1^j, \phi_2^j$ be arbitrary formulas defining the functions $\lambda_1^j$ and $\lambda_2^j$.
Let the formula $\phi^1 \land \cdots \land \phi^{n/k}$ be an arbitrary formula in which $\phi^j$ denotes either the occurrence of $\phi_1^j$ or the occurrence of $\phi_2^j$.
It is easy to see that the following fact holds.
\label{constr:1}
\end{constr}

\begin{proposition}
The set of all $2^{n/k}$ possible formulas $\phi^1 \land \cdots \land \phi^{n/k} \land \widetilde{C}$ forms a SAT partitioning of $\widetilde{C}$.
\end{proposition}

Note that if all functions $\lambda_1^j, \lambda_2^j$ are \emph{balanced} (take values 0 and 1 on an equal number of sets of values of variables from $X^j$),
then for any set $\phi^1, \dotsc, \phi^{n/k}$ the formula $\phi^1 \land \cdots \land \phi^{n/k} \land C_{f \glue h}$ has $2^{(1 - 1/k)/n}$ satisfying assignments.
For example, let us suppose that $n$ is even and consider the following case: $X^1 = \{x_1, x_2\}, X^2 = \{x_3, x_4\}, \dotsc, X^{n/2} = \{x_{n-1}, x_n\}$.
We choose the functions $\lambda_1^j(u, v) = (u \xor v)$, $\lambda_2^j(u, v) = (u \equiv v)$.
Then, each formula $\phi^1 \land \cdots \land \phi^{n/2} \land C_{f \glue h}$ has $2^{n/2}$ satisfying assignments.

The second class of SAT partitionings which gives good results on LEC instances uses specially constructed cubes over subsets $B \subset X \setminus X^{in}$.

\begin{constr}
Let us start again with LEC for two circuits $S_f, S_h$ and consider CNF formulas $C_f, C_h, C_{f \glue h}, C_{f \glue h} \land C(\mu)$.
Let $X_f, X_h$ be the sets of variables occurring in formulas $C_f, C_h$ accordingly.
Consider the following sets: $B_f \subset X_f \setminus X^{in}$, $B_h \subset X_h \setminus X^{in}$, $\tilde{B} = B_f \cup B_h = \{\tilde{x}_1, \dotsc, \tilde{x}_s\}$ (we assume that $|\tilde{B}| = s$).
Obviously, the set of cubes $\Pi = \{\tilde{x}_1^{\alpha_1} \land \cdots \land \tilde{x}_s^{\alpha_s}\}_{\alpha \in \{0, 1\}^{|\tilde{B}|}}$, $\alpha = (\alpha_1, \dotsc, \alpha_s)$
yields a SAT partitioning of formulas $C_{f \glue h}$ and $C_{f \glue h} \land C(\mu)$.
\label{constr:2}
\end{constr}

Our next task is to learn how to build sets $\tilde{B}$ (following \cite{CP2021} we will refer to such a set as to a \emph{decomposition set}) that would provide acceptable hardness \textit{w.r.t.}~$\Pi$ for computationally hard LEC instances.
Our approach to constructing sets $\tilde{B}$ is based on the concept of a statistically balanced variable.

Below we consider the circuit $S_f$, implying that the results obtained are also applicable to $S_h$ and $S_{f \glue h}$.
Keeping in mind all notions introduced above, we define a uniform distribution on~$\{0,1\}^n$, and for an arbitrary $v \in V \setminus V^{in}$ consider the events $g_v = 0$ and $g_v = 1$.
Denote by $\Pr\{g_v = 0\}$ and $\Pr\{g_v = 1\}$ the probabilities of these events.

\begin{definition}[Balanced gate]
Let $v \in V \setminus V^{in}$ be arbitrary gate.
We call the gate $v$ and the corresponding variable $var(v)$ \emph{balanced} if ${\Pr\{g_v = 0\}} = {\Pr\{g_v = 1\}} =\nobreak 1/2$.
\end{definition}

The balance of an arbitrary gate $v \in V \setminus V^{in}$ can be estimated efficiently using Chernoff bound~\eqref{eq:cheb-3}.
Indeed, let $p_v = \Pr\{g_v = 1\}$, then $v$ is associated with a Bernoulli random variable $\xi_v$, which takes the value 1 if $g_v$ on the random input $\alpha \in \{0,1\}^n$ takes the value 1.
Since $E[\xi_v] = p_v$, then for any fixed $\varepsilon, \delta \in (0, 1)$ we can construct an $(\varepsilon, \delta)$-approximation of~$p_v$ using some sample of random input words of~$S_f$.
Recall (see \textit{e.g.}~\cite{KarpLuby89}) that an $(\varepsilon, \delta)$-approximation of some parameter $\nu$ is some observable quantity $\tilde{\nu}$ such that $\Pr\{|\nu - \tilde{\nu}| \leq \varepsilon\} \geq 1 - \delta$.
Then it follows from~\eqref{eq:cheb-3} that for any fixed $\varepsilon, \delta \in (0, 1)$ to obtain the $(\varepsilon, \delta)$-approximation of $p_v$, we only need to make $N \geq \left\lceil \frac{4 \, \ln(2/\delta)}{\varepsilon^2} \right\rceil$ independent observations $\xi_1, \dotsc, \xi_N$ of random variable $\xi_v$, and compute the value $\frac{1}{N} \bigsumnolim{j = 1}^N \xi_j$.
This can be done efficiently.
For example, for $\varepsilon = 0.05$ and $\delta = 0.01$, the value $\frac{1}{N} \bigsumnolim{j = 1}^N \xi_j$ gives the required approximation for any $N \geq 8478$.

\section{Experiments}

\subsection{Considered tests}

In the role of $f$ and $h$ we considered functions defined by various algorithms that sort $k$ arbitrary natural numbers represented by bit vectors of length~$l$.
Thus, we considered  $f, h \colon \{0,1\}^n \rightarrow \{0,1\}^n$, where $n = k  \, l$.
More specifically, three sorting algorithms were used: Bubble Sorting, Selection Sorting~\cite{cormen90} and Pancake Sorting~\cite{gates1979}.
The functions $f,h$ corresponding to these algorithms were specified using And-Inverter Graphs.
We applied ABC~\cite{BraytMisch10} to build a FRAIG (Functionally Reduced And-Inverter Graph)~\cite{MischFRAIG05} for each considered circuit.
The resulting circuits were used to construct the formulas $C_f, C_h, C_{f \glue h}, C_{f \glue h} \land C(\mu)$.
The corresponding tests are denoted as follows: \instance{BvP}{k,l} for Bubble vs. Pancake; \instance{BvS}{k,l} for Bubble vs. {Selection}; and \instance{PvS}{k,l} for Pancake vs. Selection.
It should be noted that the resulting test classes scale very well and give complex instances even for relatively small input lengths.
So, for example, the \instance{PvS}{10,4} test instance is beyond the power of any of the conventional state-of-the-art SAT solvers.
However, as we show below, they can be solved on a computing cluster in reasonable time using the partitionings described above.

\subsection{Experimental setup and implementation details}
In computational experiments, we used SAT solvers that ranked best in SAT competition and SAT Race of recent years: Kissat~\cite{kissat-cadical}, CaDiCaL~\cite{kissat-cadical}, and \verb/MalpleLCMDistChronoBT-DL/~\cite{maple-lcm}.
To implement the SAT partitioning strategies described above, an MPI application was written in Python.
This program was run on a computing cluster ``Academician V.M. Matrosov''\footnote{Irkutsk Supercomputer Center of SB RAS, \url{http://hpc.icc.ru}}, each computing node of which is equipped with two 18-core Intel Xeon E5-2695 v4 Broadwell processors with 128 GB RAM (thus, 36 cores per one node were harnessed).
When~constructing the Monte Carlo estimates, random samples of 10000 were used.
Up to ten computing nodes (360 cores in total) were used in the experiments.

We note here separately that apart from the solution time in seconds, we also measured the number of conflicts generated during SAT solving, since the number of conflicts can be considered as an estimate of the size of the search tree that the SAT solver explores when solving a specific instance.
Indeed, the operation of the DPLL algorithm~\cite{dpll-2,dpll-1} corresponds to an ordinary binary tree, each branch of which, for an unsatisfiable test, ends in a conflict.
In the case of CDCL, due to periodic restarts, instead of a tree, we are dealing with a forest.
The number of paths in such a forest, in fact, can be considered as the complexity of a specific unsatisfiability proof that the SAT solver builds for the instance in question.

\subsection{Main experimental results}

For each series of tests \Instance{BvP}, \Instance{BvS}, and \Instance{PvS}, we generated and solved families of LEC instances of increasing complexity, corresponding to the following parameters: $k \in \{7, 8, 9, 10\}$, $l = 4$.
\cref{tab:single-thread-times} shows the time used to solve these instances on one cluster one using one thread.
We only included hard instances that were solved in more than three hours.
The notation \emph{>3d} means that the corresponding instance was not solved in three days and the computation was interrupted. 
Also, since the Maple solver, if interrupted, does not output the number of generated conflicts, the corresponding data is omitted.
In the next series of experiments, SAT partitionings were built in accordance with~\cref{constr:1} and~\cref{constr:2}.

\begin{table}[!htb]    
    \centering
    \caption{Time (in seconds) and number of conflicts used by sequential SAT solvers on considered instances}
    \label{tab:single-thread-times}
\begin{tabular}{
    r
    rr
    rr
    rr
}
    \toprule
    & \multicolumn{2}{c}{Kissat} & \multicolumn{2}{c}{Cadical} & \multicolumn{2}{c}{Maple} \\
    \thead{Instance} & \thead{Time} & \thead{Confl.} & \thead{Time} & \thead{Confl.} & \thead{Time} & \thead{Confl.} \\
    \midrule
    $BvP_{9,4}$ & \numprint{11316} & $9 \cdot 10^7$ & \numprint{25710} & $9 \cdot 10^7$ & \numprint{86389} & $18 \cdot 10^7$ \\
    $BvP_{10,4}$ & \numprint{154410} & $62 \cdot 10^7$ & \numprint{246294} & $47 \cdot 10^7$ & >3d & --- \\
    $BvS_{9,4}$ & \numprint{3054} & $27 \cdot 10^6$ & \numprint{5478} & $28 \cdot 10^6$ & \numprint{8564} & $31 \cdot 10^6$ \\
    $BvS_{10,4}$ & \numprint{14272} & $97 \cdot 10^6$ & \numprint{36048} & $109 \cdot 10^6$ & \numprint{57964} &	$130 \cdot 10^6$ \\
    $PvS_{9,4}$ & \numprint{64437} & $28 \cdot 10^7$ & \numprint{108025} & $26 \cdot 10^7$ & >3d & --- \\
    $PvS_{10,4}$ & >3d & $73 \cdot 10^7$ & >3d & $40 \cdot 10^7$ & >3d & --- \\
    \bottomrule
\end{tabular}

\end{table}

In the case of~\cref{constr:1}, we used a partitioning of the set $X^{in}$ into disjoint pairs, triples, and quadruples of variables. 
Function~$\lambda_1^j$ used in~\cref{constr:1} was selected experimentally as follows ($\lambda_2^j = \lnot \lambda_1^j$ in all cases): 
\begin{itemize}
    \item 2-XOR: for pairs, function $\lambda_1^j$, $j \in \{1, \dotsc, n/2\}$ was defined by formula $\phi_1^j = (x_1^j \xor x_2^j)$;
    \item 3-MAJ: for triples, function $\lambda_1^j$, $j \in \{1, \dotsc, n/3\}$ was defined by formula $\phi_1^j = \com{majority}(x_1^j, x_2^j, x_3^j)$, where $\com{majority}(a, b, c) = a \land b \vee a \land c \vee b \land c$;
    \item 4-BENT: for quadruples, $\lambda_1^j$, $j \in \{1, \dotsc, n/4\}$ was defined by the bent function~\cite{tokareva2015} of four variables according to the formula $\phi_1^j = x_1^j \land x_3^j \xor x_2^j \land x_4^j$.
\end{itemize}
Note that partitioning into pairs produces a large number of subproblems, which, although simple, result in much higher estimates of the total solving time than for triples and quadruples.

\begin{table*}[!tb]
    \tiny
    \centering    
    \caption{Experimental results for solving decompositions of LEC instances \instance{BvP}{9,4}, \instance{BvP}{10,4}, \instance{PvS}{9,4}}
    \label{tab:main}
    
\setlength{\tabcolsep}{3pt}
\newcommand{\supmidrule}{\cmidrule(lr){1-12}}
\newcommand{\submidrule}{\cmidrule(lr){4-12}}

\begin{tabular}{
    c
    c
    c
    c
    c
    r@{}c@{}l
    r@{}c@{}l
    r
}
\toprule
    \thead{Instance} & \thead{Dec. \\ type} & \thead{Dec. \\ size} & \thead{Sample\\ size} & \thead{Solver} & \multicolumn{3}{c}{\thead{Avg $\pm$ sd \\ time, s}} & \multicolumn{3}{c}{\thead{Avg $\pm$ sd \\ confl.}} & \thead{Wall \\ time, s} \\

\midrule
    \instance{BvP}{9,4}
    & 2-XOR & \numprint{262144}
      & \numprint{10000} & Cadical & 19 & $\,\pm\,$ & 4 & $190 \cdot 10^{3}$ & $\,\pm\,$ & $34 \cdot 10^{3}$ & --- \\
    &&& \numprint{10000} & Kissat & 21 & $\,\pm\,$ & 5 & $259 \cdot 10^{3}$ & $\,\pm\,$ & $57 \cdot 10^{3}$ & --- \\
    &&& \numprint{10000} & Maple & 114 & $\,\pm\,$ & 19 & $411 \cdot 10^{3}$ & $\,\pm\,$ & $52 \cdot 10^{3}$ & --- \\
\submidrule
    & 3-MAJ & \numprint{4096}
      & \numprint{4096} & Cadical & 355 & $\,\pm\,$ & 109 & $\numprint{2218} \cdot 10^{3}$ & $\,\pm\,$ & $542 \cdot 10^{3}$ & \numprint{8087} \\
    &&& \numprint{4096} & Kissat & 276 & $\,\pm\,$ & 76 & $\numprint{2808} \cdot 10^{3}$ & $\,\pm\,$ & $709 \cdot 10^{3}$ & \numprint{6286} \\
    &&& \numprint{4096} & Maple & 797 & $\,\pm\,$ & 216 & $\numprint{2592} \cdot 10^{3}$ & $\,\pm\,$ & $635 \cdot 10^{3}$ & \numprint{18132} \\
\submidrule
      & 4-BENT & 512
      & 512 & Cadical & \numprint{2214} & $\,\pm\,$ & \numprint{1149} & $10 \cdot 10^{6}$ & $\,\pm\,$ & $5 \cdot 10^{6}$ & \numprint{6299} \\
    &&& 512 & Kissat & \numprint{1168} & $\,\pm\,$ & \numprint{447} & $12 \cdot 10^{6}$ & $\,\pm\,$ & $5 \cdot 10^{6}$ & \numprint{3323} \\
    &&& 512 & Maple & \numprint{4273} & $\,\pm\,$ & \numprint{1923} & $11 \cdot 10^{6}$ & $\,\pm\,$ & $5 \cdot 10^{6}$ & \numprint{12153} \\
\submidrule
    & 4+4 & 256
      & 256 & Cadical & \numprint{1358} & $\,\pm\,$ & \numprint{540} & $9 \cdot 10^{6}$ & $\,\pm\,$ & $3 \cdot 10^{6}$ & \numprint{1931} \\
    &&& 256 & Kissat & \numprint{884} & $\,\pm\,$ & \numprint{323} & $11 \cdot 10^{6}$ & $\,\pm\,$ & $4 \cdot 10^{6}$ & \numprint{1258} \\
    &&& 256 & Maple & \numprint{2286} & $\,\pm\,$ & \numprint{836} & $9 \cdot 10^{6}$ & $\,\pm\,$ & $3 \cdot 10^{6}$ & \numprint{3252} \\

\supmidrule
    \instance{BvP}{10,4}
    & 3-MAJ & \numprint{16384}
      & \numprint{10000} & Cadical & \numprint{1752} & $\,\pm\,$ & \numprint{886} & $7 \cdot 10^{6}$ & $\,\pm\,$ & $3 \cdot 10^{6}$ & --- \\
    &&& \numprint{10000} & Kissat & \numprint{1072} & $\,\pm\,$ & \numprint{476} & $9 \cdot 10^{6}$ & $\,\pm\,$ & $4 \cdot 10^{6}$ & --- \\
\submidrule
    & 4-BENT & \numprint{1024}
      & \numprint{1024} & Cadical & \numprint{22397} & $\,\pm\,$ & \numprint{15010} & $63 \cdot 10^{6}$ & $\,\pm\,$ & $30 \cdot 10^{6}$ & \numprint{127415} \\
    &&& \numprint{1024} & Kissat & \numprint{10472} & $\,\pm\,$ & \numprint{5667} & $77 \cdot 10^{6}$ & $\,\pm\,$ & $35 \cdot 10^{6}$ & \numprint{59571} \\
\submidrule
    & 4+4 & 256
      & 256 & Cadical & \numprint{45494} & $\,\pm\,$ & \numprint{12845} & $102 \cdot 10^{6}$ & $\,\pm\,$ & $23 \cdot 10^{6}$ & \numprint{64703} \\
    &&& 256 & Kissat & \numprint{18155} & $\,\pm\,$ & \numprint{6451} & $124 \cdot 10^{6}$ & $\,\pm\,$ & $37 \cdot 10^{6}$ & \numprint{25821} \\

\supmidrule
    \instance{PvS}{9,4}
    & 3-MAJ & \numprint{4096}
      & \numprint{4096} & Cadical & \numprint{941} & $\,\pm\,$ & \numprint{317} & $4 \cdot 10^{6}$ & $\,\pm\,$ & $1 \cdot 10^{6}$ & \numprint{21422} \\
    &&& \numprint{4096} & Kissat & \numprint{766} & $\,\pm\,$ & \numprint{229} & $6 \cdot 10^{6}$ & $\,\pm\,$ & $2 \cdot 10^{6}$ & \numprint{17443} \\
\submidrule
    & 4-BENT & 512
      & 512 & Cadical & \numprint{6491} & $\,\pm\,$ & \numprint{3512} & $20 \cdot 10^{6}$ & $\,\pm\,$ & $10 \cdot 10^{6}$ & \numprint{18462} \\
    &&& 512 & Kissat & \numprint{3831} & $\,\pm\,$ & \numprint{1641} & $27 \cdot 10^{6}$ & $\,\pm\,$ & $10 \cdot 10^{6}$ & \numprint{10898} \\
\submidrule
    & 6+6 & \numprint{4096}
      & \numprint{4096} & Cadical & \numprint{421} & $\,\pm\,$ & \numprint{492} & $2 \cdot 10^{6}$ & $\,\pm\,$ & $2 \cdot 10^{6}$ & \numprint{9588} \\
    &&& 4096 & Kissat & \numprint{390} & $\,\pm\,$ & \numprint{465} & $4 \cdot 10^{6}$ & $\,\pm\,$ & $4 \cdot 10^{6}$ & \numprint{8869} \\

\bottomrule
\end{tabular}
	
\end{table*}

In \cref{tab:main}, columns \emph{Total subprobs} and \emph{Solved subprobs} contain information about the total number of subproblems in the SAT partitioning (column \emph{Decomposition Type}) and the number of subproblems solved in the experiment using five nodes of a computing cluster (180 cores).
If the values in these columns are equal, then it means that all subproblems from the corresponding SAT partitioning have been solved.
In these cases we compute the exact value of $E[\xi_\Pi]$ and the standard deviation of this value, both in seconds and in the number of conflicts.
Otherwise, if the number of solved subproblems is smaller than the total number of subproblems, then we present statistical estimates of these values, calculated using the specified sample size.
The column \emph{Wall clock time} shows the time used to solve the corresponding partitioning: it corresponds to the time the user would need to wait in order to solve the LEC instance using the said partitioning.
If the number of solved subproblems (i.e. the sample size) is smaller than the total number of subproblems, this value is omitted.

In the experiments for~\cref{constr:2}, we used cubes built from variables corresponding to balanced gates (we refer to such variables and cubes as to \emph{balanced} ones).
More precisely, for each circuit $S_f$ and $S_h$, the balance of each gate was calculated in the 
manner described above:  
in fact we constructed (using Chernoff bound) 
$(\varepsilon,\delta)$ approximations of probability 
$\Pr\{g_v =1\}$ with $\varepsilon=0.05$ and $\delta=0.01$. 
Then, from each circuit we chose $q$ gates with this estimation closest to $1/2$, and built the decomposition set $B = \{\tilde{x}_1, \dotsc, \tilde{x}_{2q}\}$ from the obtained variables.
The considered SAT partitioning (denoted as $q$+$q$) is represented by all possible cubes $\tilde{x}_1^{\alpha_1} \land \cdots \land \tilde{x}_{2q}^{\alpha_{2q}}$.
The experiments were carried out for $q \in \{4, 5, 6\}$.

In the context of all that has been said above, one of the main issues is the accuracy of the resulting estimates of~$E[\xi_\Pi]$.
The main factor that negatively affects the accuracy is the magnitude of~$Var(\xi_\Pi)$.
The data in \cref{tab:main} implies that the two proposed SAT partitioning constructions give a relatively small standard deviation and, as a result, the resulting estimates are very accurate.


Moreover, as shown below, in order to obtain relatively accurate estimates of~$E[\xi_\Pi]$, it is sufficient to use samples whose size is significantly smaller than the total size of the considered SAT partitioning.
The aforesaid is confirmed by the experimental data shown in~\cref{fig:min-max-BvP-9-4-cadical-triples} and \cref{fig:min-max-BvP-9-4-cadical-4+4}, which demonstrate the dependence of the accuracy of the estimate of~$E[\xi_\Pi]$ on the size of the random sample.
In~\cref{fig:min-max-BvP-9-4-cadical-triples} we present the plot for partitioning of the set~$X^{in}$ into triples (3-MAJ) for \cref{constr:1}, and in \cref{fig:min-max-BvP-9-4-cadical-4+4} into balanced cubes~(4+4) for \cref{constr:2}.
In both cases we used the LEC problem instance \instance{BvP}{9,4} and the solver CaDiCaL.


For each value of the size of random sample~$N$ we generated $P = 1000$ random samples of size~$N$ and calculated the sample means $(\smash{\hat{\xi}}^1, \dotsc, \smash{\hat{\xi}}^P)$, where each $\smash{\hat{\xi}^r} = \frac{1}{N} \bigsumnolim{j=1}^{N} \xi_j, r\in \{1,\ldots,P\}$.
Additionally, we calculated the mean of sample means $\Xi(N) = \frac{1}{P} \bigsumnolim{r=1}^{P} \smash{\hat{\xi}}^r$, and also chose the minimal $M_*(N)$ and maximal $M^*(N)$ values.
Next, we normalized all values by dividing them by $E[\xi_\Pi]$.

\begin{figure}[!htb]
    \centering
    \includegraphics{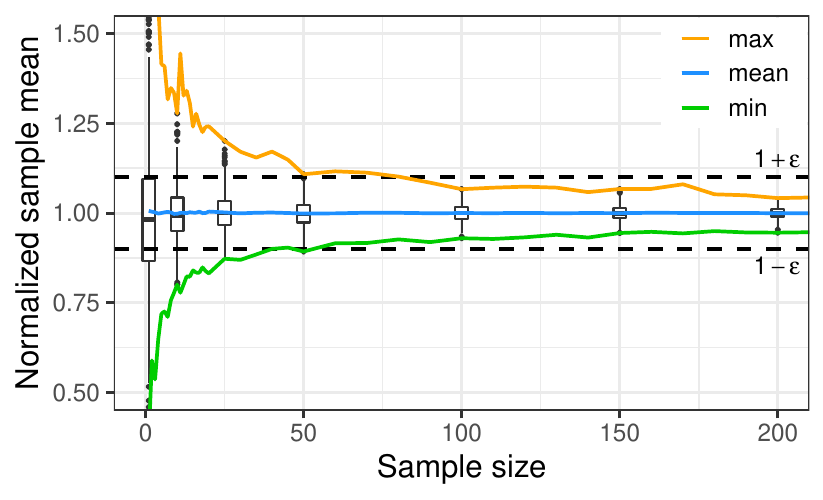}
    \caption{Distributions of sample means for different sample sizes~$N$ on the 3-MAJ decomposition of \instance{BvP}{9,4} instance}
    \label{fig:min-max-BvP-9-4-cadical-triples}
\end{figure}

\begin{figure}[!htb]
    \centering
    \includegraphics{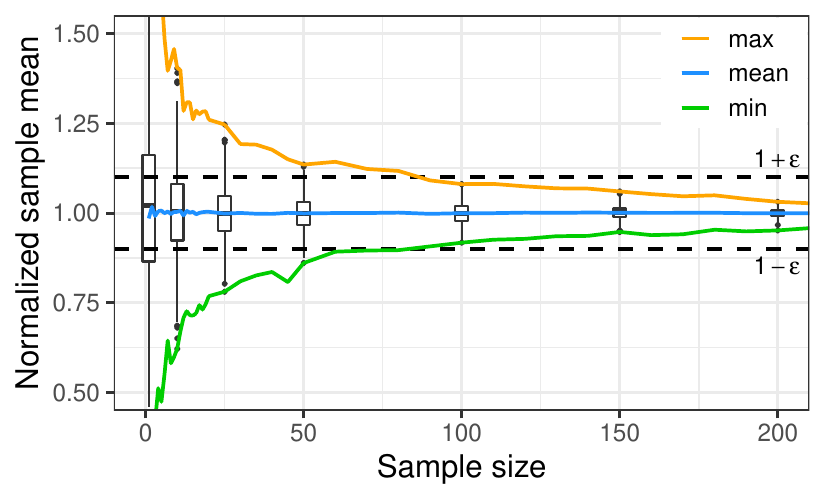}
    \caption{Distributions of sample means for different sample sizes~$N$ on the decomposition into balanced cubes (4+4) of \instance{BvP}{9,4} instance}
    \label{fig:min-max-BvP-9-4-cadical-4+4}
\end{figure}

In \cref{fig:min-max-BvP-9-4-cadical-triples} and \cref{fig:min-max-BvP-9-4-cadical-4+4} the horizontal axis shows the varying size of random sample~$N$.
For some values of~$N$, the corresponding distributions of sample means are shown using boxplots.
Additionally, the plots contain the following normalized lines:
\begin{itemize}
    \item $\Xi(N) / E[\xi_\Pi]$ (blue line, middle);
    \item $M_*(N) / E[\xi_\Pi]$ (green line, bottom);
    \item $M^*(N) / E[\xi_\Pi]$ (orange line, top);
    \item $(1 \pm \varepsilon)$ for $\varepsilon = 0.1$ (black dashed lines).
\end{itemize}




From the plots in \cref{fig:min-max-BvP-9-4-cadical-triples}--\cref{fig:min-max-BvP-9-4-cadical-4+4} it can be seen that on the considered class of tests, the calculated sample mean $\smash{\hat{\xi}}$ gives a fairly accurate estimate of~$E[\xi_\Pi]$ even when the sample size~$N$ is significantly smaller than the total size
of the considered partitioning.

\begin{table*}[t]
    \tiny
    \centering   
    
    \caption{Experimental results for solving decompositions on \instance{PvS}{10,4} instance}
    \label{tab:exp-pvs-10-4}
    \setlength{\tabcolsep}{3pt}
    \begin{tabular}{
        c
        c
        c
        c
        r@{}c@{}l
        r@{}c@{}l
        r
    }
        \toprule
        \thead{Dec. \\ type} & \thead{Dec.\\ size} & \thead{Sample \\ size} & \thead{Solver} & \multicolumn{3}{c}{\thead{Avg. $\pm$ sd \\ time, s}} & \multicolumn{3}{c}{\thead{Avg. $\pm$ sd \\ conflicts}} & \thead{Wall clock \\ time, s} \\
        \midrule

        2-XOR & \numprint{1048576} 
        & \numprint{10000} & Cadical
            & 167 & $\,\pm\,$ & 57
            & $1011  \, 10^{3}$ & $\,\pm\,$ & $240  \, 10^{3}$
            & --- \\
        &
        & \numprint{10000} & Kissat
            & 185 & $\,\pm\,$ & 64
            & $1520  \, 10^{3}$ & $\,\pm\,$ & $393  \, 10^{3}$
            & --- \\
        \midrule
        3-MAJ & \numprint{16384}
        & \numprint{10000} & Cadical
            & $5  \, 10^{3}$ & $\,\pm\,$ & $3  \, 10^{3}$
            & $15  \, 10^{6}$ & $\,\pm\,$ & $7  \, 10^{6}$
            & --- \\
        \midrule
        4-BENT & \numprint{1024}
        & \numprint{1024} & Cadical
            & $83  \, 10^{3}$ & $\,\pm\,$ & $48  \, 10^{3}$
            & $111  \, 10^{6}$ & $\,\pm\,$ & $47  \, 10^{6}$
            & \numprint{474922} \\
        &
        & \numprint{1024} & Kissat
            & $30  \, 10^{3}$ & $\,\pm\,$ & $13  \, 10^{3}$
            & $129  \, 10^{6}$ & $\,\pm\,$ & $43  \, 10^{6}$
            & \numprint{171182} \\
        \midrule
        6+6 & \numprint{4096}
        & \numprint{4096} & Kissat
            & $3  \, 10^{3}$ & $\,\pm\,$ & $19  \, 10^{3}$
            & $10  \, 10^{6}$ & $\,\pm\,$ & $56  \, 10^{6}$
            & \numprint{71168} \\
        \midrule
        4+4 & 256
        & 256 & Kissat
            & $26  \, 10^{3}$ & $\,\pm\,$ & $74  \, 10^{3}$
            & $68  \, 10^{6}$ & $\,\pm\,$ & $59  \, 10^{6}$
            & \numprint{37606} \\
        
        \bottomrule
    \end{tabular}
\end{table*}

As mentioned above, problems from the considered class with input length $n = k  \, l = 40$ are already extremely complex.
However, problems \instance{BvP}{10,4} and \instance{BvS}{10,4} were solved using five nodes (180 cores) of the computing cluster in reasonable time (as can be seen in \cref{tab:main}).
Since the obtained estimates of hardness for \instance{PvS}{10,4} were significantly higher than for \instance{BvP}{10,4} and \instance{BvS}{10,4}, we used ten cluster nodes (360 cores), CaDiCaL and Kissat solvers to solve them (Maple showed significantly worse results in previous experiments). 
Results are shown in~\cref{tab:exp-pvs-10-4}.

\subsection{Experiments with unbalanced cubes}
We emphasize that in \cref{constr:2} we use cubes built from the most balanced variables, hoping that the corresponding SAT partitioning will have a small variance.
And this hypothesis, as follows from \cref{tab:main}, is generally confirmed.
Of interest is the question of what will happen if we build cubes using the most \emph{unbalanced} variables instead of balanced ones (\ie unbalanced cubes)?
On the one hand, $Var(\xi_\Pi)$ should be significantly higher, but, on the other hand, many subproblems in the SAT partitioning can be extremely simple.

We have carried out the corresponding experiments.
It turned out that when using unbalanced cubes, in many cases even the CNF formulas $\tilde{x}_1^{\alpha_1} \land \cdots \land \tilde{x}_{2q}^{\alpha_{2q}} \land C_{f \glue h}$ are unsatisfiable, \ie formulas that do not even include the term~$C(\mu)$ which encodes the miter.
And the corresponding instances are easy for the SAT solver.
However, the final SAT partitioning will necessarily contain abnormally hard formulas, the hardness of which is comparable with the hardness of SAT for $C_{f \glue h} \land C(\mu)$ (\ie for the case without using partitioning).

Let us denote as~$G^*$ such an abnormally hard cube. 
Using Chernoff bound, we estimated the number of satisfying assignments of CNF formula $G^* \land C_{f \glue h}$. 
We conducted several such experiments with hard cubes, and the typical case is: for a hard cube $G^{*}$ of realistic size (say, $\leq 40$) the estimation of the number of satisfying assignments of formula $G^{*} \land C_{f \glue h}$ was greater than $0.9  \, 2^n$ (with tolerance $\varepsilon = 0.01$ and confidence level $1-\delta = 0.99$, \textit{w.r.t.} Chernoff bound, $N = 211933$ was used).
Thus, these results confirm the assumption put forward above about a direct relationship between the number of satisfying assignments of formula $G \land C_{f \glue h}$ and the hardness of formula $G \land C_{f \glue h} \land C(\mu)$.

\section{Conclusion}

In this paper, we explored how to estimate the hardness of SAT encodings for the Logical Equivalence Checking problem. 
One of our basic observations in this context is that we can estimate the hardness of a SAT encoding of LEC using some SAT partitioning. 
More specifically, we introduce the concept of hardness of a SAT instance \emph{w.r.t.} a SAT partitioning and a SAT solver $\Oracle$. 
We show that such estimates can be constructed using probabilistic algorithms based on the Monte Carlo method. 
The accuracy of this kind of estimates depends on the probabilistic characteristics of a specially defined random variable which is associated with a particular SAT partitioning. 
We propose two constructions of SAT partitionings, in relation to which we present arguments for the good accuracy of the obtained estimates of hardness. 
To carry out computational experiments, we use a class of LEC instances, where circuits are represented as And-Inverter Graphs which define various algorithms for sorting $k$ natural numbers with bit length~$l$.
The hardness of such tests scales well due to the selection of values $k, l$, and one can generate extremely hard LEC instances already for circuits with $n = k  \, l = 40$ inputs. 
In general, it is not possible to accurately predict the running time of a consecutive SAT solver on some of these tests.
However, we estimate the hardness of such tests \emph{w.r.t.} the proposed SAT partitioning. 
The estimates obtained indicate that the corresponding LEC instances can be solved in parallel using a reasonable amount of computational resources. 
We confirm these conclusions and the accuracy of the estimates obtained by solving the corresponding instances on a computing cluster. 
We also formulate a hypothesis about a direct relationship between the hardness of subproblems in the SAT partitioning and the number of satisfying assignments of special satisfiable CNF formulas associated with the original Boolean circuits, and we demonstrate that this hypothesis is true for circuits considered in our experiments.

Future work may include the development of optimization algorithms similar to ones described in \cite{CP2021,Zaikin2021,Semenov2021-bb} for finding cubes over auxiliary variables with good statistical estimation of hardness of LEC \emph{w.r.t.} the corresponding SAT partitioning.

\section*{Acknowledgements}
The research is supported by Huawei (grant TC20211213625).

\bibliographystyle{splncs03}


\end{document}